\documentclass{article}

\usepackage{makecell}

\usepackage{wrapfig}
\usepackage{bm}
\usepackage{graphicx}
\usepackage{todonotes}
\usepackage{color}
\usepackage[font=small,labelfont=bf]{caption}
\newcommand\mynotes[1]{#1}%\textcolor{red}{#1}}
\usepackage[accepted]{icml2019} %[preprint]

\usepackage[utf8]{inputenc} % allow utf-8 input
\usepackage[T1]{fontenc}    % use 8-bit T1 fonts
\usepackage{hyperref}       % hyperlinks
\usepackage{url}            % simple URL typesetting
\usepackage{booktabs}       % professional-quality tables
\usepackage{amsfonts}       % blackboard math symbols
\usepackage{nicefrac}       % compact symbols for 1/2, etc.
\usepackage{microtype}      % microtypography
\usepackage{amsmath}
 \usepackage{amssymb} 
 \usepackage{amsthm}
 \usepackage{subcaption} 
\newtheorem{prop}{Proposition}[section]
\usepackage{ulem}
\usepackage{lipsum}
\usepackage{lineno}
\icmltitlerunning{Greedy Layerwise Learning Can Scale to ImageNet}
\hypersetup{draft}
\begin{document}
\twocolumn[
\icmltitle{Greedy Layerwise Learning Can Scale to ImageNet}
           
\begin{icmlauthorlist}
  \icmlauthor{Eugene Belilovsky}{mila}
  \icmlauthor{Michael Eickenberg}{ucb}
  \icmlauthor{Edouard Oyallon}{cd}
\end{icmlauthorlist}
\icmlaffiliation{mila}{Mila, University of Montreal}
\icmlaffiliation{ucb}{University of California, Berkeley}
\icmlaffiliation{cd}{CentraleSupelec, University of Paris-Saclay}

\icmlcorrespondingauthor{Edouard Oyallon}{edouard.oyallon@centralesupelec.fr}
\icmlcorrespondingauthor{Michael Eickenberg}{michael.eickenberg@berkeley.edu}
\icmlcorrespondingauthor{Eugene Belilovsky}{eugene.belilovsky@umontreal.ca}

\icmlkeywords{Machine Learning, ICML}

\vskip 0.3in
]
\printAffiliationsAndNotice{} 

\begin{abstract}

Shallow supervised 1-hidden layer neural networks have a number of favorable properties that make them easier to interpret, analyze, and optimize than their deep counterparts, but lack their representational power. 
Here we use 1-hidden layer learning problems to sequentially build deep networks layer by layer, which can inherit 
%some 
properties from shallow networks. 
%
%Here, we ask wether shallow neural networks have the ability to capture important properties of deep neural networks.
%
%We combine the strengths of both approaches by building deep networks layerwise, using 1-hidden-layer learning problems, inheriting theoretical properties as well as representational power.
%Similar learning ideas have been suggested, but we ask here if such a strategy can compete on problems where deep learning is critical for success. 
Contrary to previous approaches using shallow networks, we focus on problems where deep learning is reported as critical for success. 
We thus study CNNs on image 
%recognition 
classification
tasks using the large-scale ImageNet dataset and the CIFAR-10 dataset.
%Using a simple set of ideas tailored to CNNs
Using a simple set of ideas for architecture and training we find that solving sequential 1-hidden-layer auxiliary problems lead to a CNN that exceeds AlexNet performance on ImageNet. 
Extending this training methodology to construct individual layers by solving 2-and-3-hidden layer auxiliary problems, we obtain an 11-layer network that exceeds several members of the VGG model family on ImageNet, and can train a VGG-11 model to the same accuracy as end-to-end learning. To our knowledge, this is the first competitive alternative to end-to-end training of CNNs that can scale to ImageNet. We illustrate several interesting properties of these models theoretically and %empirically 
conduct a range of experiments to study the properties this training induces on the intermediate representations.
\end{abstract}
\vspace*{-25pt}    % YOOOOOOO I moved the title of intro slightly upwards :)
\section{Introduction}
\label{sec:intro}
%%% DEEP LEARNING HAS HAD A LOT OF SUCCESS ON IMAGES, THEN ON OTHER THINGS. BUT IT IS HARD TO ... UNDERSTAND/MANIPULATE (WHAT?)EB: what is this troll comment?
Deep Convolutional Neural Networks (CNNs) trained on large-scale supervised data via the back-propagation algorithm have become the dominant approach in most computer vision tasks \citep{krizhevsky2012imagenet}. 
This has motivated successful applications of deep learning in other fields such as speech recognition \citep{chan2016listen}, natural language processing \citep{vaswani2017attention}, and reinforcement learning \citep{silver2017mastering}. 
%However, understanding the behavior of 
%end-to-end trained 
%deep networks and how they achieve their remarkable performance has remained persistently elusive. 
%One reason for this difficulty is the end-to-end training of the layers. 
Training procedures and architecture choices for deep CNNs have become more and more entrenched, but which of the standard components of modern pipelines are essential to the success of deep CNNs is not clear. Here we %pose the question: 
ask:
do CNN layers need to be learned jointly to obtain high performance? We will show that even for the %complex 
challenging
ImageNet dataset the answer is \textit{no}. 

Supervised end-to-end learning is the standard approach to neural network optimization.
However it has potential issues that can be valuable to consider.
%Supervised end-to-end learning is the standard approach to neural network optimization, which can work well when all settings are chosen right. however potential issues remain.
% First, the use of a global objective means that the final functional behavior of intermediate layers of a deep network is poorly understood: it is unclear how they work together to achieve their high-accuracy predictions. 
First, the use of a global objective means that the final functional behavior of individual intermediate layers of a deep network is only indirectly specified: % unconstrained:
it is unclear how the layers work together to achieve high-accuracy predictions.
Several authors have suggested and shown empirically that CNNs learn to implement mechanisms 
%with respect to depth 
%that progressively become invariant 
that progressively induce invariance
to complex, but irrelevant variability \citep{mallat2016understanding,yosinski2015understanding} 
%and increase 
while increasing
linear separability \citep{zeiler2014visualizing,oyallon2017building,jacobsen2018revnet} of the data.
%
%In the latter case it 
Progressive linear separability has been shown empirically but it is unclear whether this is merely the consequence of other strategies implemented by CNNs, 
or if it is a sufficient condition for the %observed 
high performance of these networks.
Secondly, understanding the link between shallow Neural Networks (NNs) and deep NNs is difficult: while generalization, approximation, or optimization  results~\citep{barron1994approximation,bach2014breaking,venturi2018neural,neyshabur2018towards,pinkus1999approximation} for 1-hidden layer NNs are available, the same studies conclude that multiple-hidden-layer NNs are much more difficult to tackle theoretically.
Finally, %relying purely on the gradients as in 
end-to-end back-propagation can be inefficient \citep{jaderberg2016decoupled, salimans2017evolution} in terms of computation and memory resources and is considered not biologically plausible. 
%Moreover, for some  learning problems, the full gradient is less informative than other alternatives \citep{shalev2017failures}.%Finally, there exists a clear gap between the theory of deep learning and the practice: generalization, approximation, or optimization  results~\citep{barron1994approximation,bach2014breaking,venturi2018neural,pinkus1999approximation} for 1-hidden layer neural networks are available whereas the same studies conclude that multiple hidden layer neural networks are much more difficult to tackle.

Sequential learning of CNN layers by solving shallow supervised learning problems is an alternative to end-to-end back-propagation.
This classic \cite{Ivanhenko} learning principle can directly specify the objective of every layer. 
%for example by encouraging 
It can encourage
the refinement of specific properties of the representation~\citep{greff2016highway}, such as progressive linear separability.
%implement the intuition that deep networks operate by learning to progressively refine representations \citep{greff2016highway}, . 
%This  strategy can directly specify the objective of every layer and implement the intuition that deep networks operate by learning to progressively refine representations \citep{greff2016highway}, for example by encouraging progressive linear separability. 
The 
%complexity of developing 
development of
theoretical tools for deep greedy methods 
%will
can naturally
%be largely delegated to 
draw from the
theoretical understanding of shallow sub-problems.
Indeed, \cite{AroraBMM17,bengio2006convex,bach2014breaking,janzamin2015beating} show global optimal approximations, while other works have shown that networks based on sequential 1-hidden layer training can have a variety of guarantees under certain assumptions \citep{huang2017learning,malach2018provably,arora2014provable}: 
greedy layerwise methods could permit to cascade those results
to bigger architectures.
Finally, a greedy approach will rely much less on 
%obtaining 
having access to
a full gradient.
This can have a number of benefits.
From an algorithmic perspective, they do not require storing most of the intermediate activations nor to compute most intermediate gradients.
This can be beneficial in memory-constrained settings. Unfortunately, prior work has not convincingly demonstrated that layer-wise training strategies can tackle the sort of large-scale problems that have brought deep learning into the spotlight. 

Recently  multiple works have demonstrated interest in determining whether alternative training methods \cite{Xiao2019,bartunov2018assessing} can scale to large data-sets that have only been solved by 
deep learning. 
As is the case for many algorithms (not just training strategies) many of these alternative training strategies have been shown to work on smaller datasets (e.g. MNIST and CIFAR) but fail completely on large-scale datasets (e.g. ImageNet). 
Also, 
%the above works 
these works
largely focus on avoiding the weight transport problem in backpropagation \cite{bartunov2018assessing} while simple  greedy layer-wise learning 
%bypasses this particular 
reduces the extent of this
problem 
%to a degree 
and should be considered as a potential baseline.

In this context, our contributions are as follows. 
\textbf{(a)} First, we design a simple and scalable supervised approach to learn layer-wise CNNs in Sec. \ref{layerw}. 
\textbf{(b)} Then, Sec. \ref{sec:alex} demonstrates empirically that by sequentially solving 1-hidden layer problems, we can match the performance of the AlexNet on ImageNet. 
We motivate in Sec. \ref{sec:aux_prop} how this model can be connected to a body of theoretical work that tackles 1-hidden layer networks and their sequentially trained counterparts. 
\textbf{(c)} We show that layerwise trained layers exhibit a progressive linear separability property in Sec. \ref{sec:prog_exp}. 
\textbf{(d)} In particular, we use this to help motivate learning layer-wise CNN layers via shallow $k$-hidden layer auxiliary problems, with $k>1$.
Using this approach  our sequentially trained $3$-hidden layer models can reach the performance level of VGG models (Sec. \ref{sec:scaling}) and end-to-end learning.
\textbf{(e)} Finally, we suggest an approach to easily reduce the model size \textit{during training} of these networks. 

\section{Related Work}\label{sec:rel}
%\mynotes{We need to discuss " a provably correct algorithm for deep learning that actually works from shai}
%Main categories:
%greedy unsuper pretraing -
%supervised 
%
Several authors have previously 
%considered 
studied
layerwise learning.
In this section we review 
%several of the 
related works and re-emphasize the distinctions from our work. 
%For example, in image classification Fisher Vectors \citep{perronnin2015fisher,sanchez2013image} based approaches consist in a SIFT extraction followed by fisher vector encoders; other examples consist in stacking layers of an unsupervised auto-encoders  \citep{bo2013multipath,lin2014stable}; however, there exists a major gap in accuracy with end-to-end supervised methods \citet{krizhevsky2012imagenet} that we fill in this work thanks to supervision. 

Greedy unsupervised learning has been a popular topic of research in the past.
Greedy unsupervised learning of deep generative models \citep{bengio2007greedy,hinton2006fast} was shown to be effective as an initialization for deep supervised architectures.
\citet{bengio2007greedy} also considered supervised greedy layerwise learning as  \textit{initialization} of networks for subsequent end-to-end supervised learning, but this was not shown to be effective with the existing techniques at the time. 
Later work on large-scale supervised deep learning showed that modern training techniques permit avoiding layerwise initialization entirely \citep{krizhevsky2012imagenet}. 
%The use of greedy methods in unsuperivsed learning in general was also overshadowed by models that permitted effectively learning end-to-end \citep{kingma2013auto,goodfellow2014generative}.
We emphasize that the supervised layerwise learning we consider is distinct from unsupervised layerwise learning.
Moreover, here layerwise training is not studied as a \textit{pretraining} strategy, but a \textit{training} one.

% Tentatives of layerwise
Layerwise learning in the context of constructing supervised NNs has been attempted in several works. It was considered in multiple earlier works \citet{Ivanhenko,fahlman1990cascade,lengelle1996training} on very simple problems
and in a climate where deep learning was not a dominant supervised learning approach.
These works were aimed primarily at structure learning, building up architectures that allow the model to grow appropriately based on the data. Others %have also been 
works were
motivated by the avoidance of difficulties 
%in the vanishing gradient problem.
with vanishing gradients.
Similarly, \citet{cortes2016adanet} recently proposed a progressive learning method that builds a network such that the architecture can adapt to the problem, 
with theoretical contributions to structure learning,
%focus on the theory associated with the structure learning problem, 
but %do not consider 
not on
problems where deep networks are unmatched %.
in performance.
\citet{malach2018provably} also train a supervised network in a layerwise fashion, showing that their method provably generalizes for a restricted class of image models. However, the results of these model are not shown to be competitive with handcrafted approaches~\citep{oyallon2015deep}. Similarly \citep{kulkarni2017layer, MarquezNN} %recently 
revisit layerwise training, but 
%consider limited experimental settings.
in a limited experimental setting.

%Boosting techniques \citep{friedman2001greedy,freund1996experiments} are a greedy approach to supervised learning with a successful history and theoretical foundation and still represents the state of the art in some domains~\citep{chen2016xgboost}. 
%They  correspond to a sequential procedure for improving the accuracy of a residual: successive weak learners are progressively aggregated to form a strong learner.  
 ~\citet{huang2017learning} 
%attempt to combine 
combined boosting theory with  a residual architecture \citep{he2016deep} to sequentially train layers. However, results are presented for limited datasets and indicate that the end-to-end approach is often needed 
ultimately to obtain competitive results. This proposed strategy does not clearly outperform simple non-deep-learning baselines.
By contrast, we focus on settings where deep CNN based-approaches do not currently have  competitors and rely on a simpler objective function, which is found to scale well and be competitive with end-to-end approaches. 
%Though not our main motivation, our approach also permits structure learning and can be seen as the first demonstration at large scale.
%the sens that they agregate several weak learners to form a strong learner many analogy with sequential strategies for deep learning, in . %A major line of work considered greedy unsupervised learning of deep generative models .  If the CNN is deep, the result of this optimization is very non-linear, meaning that linear model are too simple 

% Model morphing.......
Another related thread %are 
is
methods which add layers to existing networks and then use end-to-end learning.
These approaches usually have different goals from ours, such as stabilizing end-to-end learned models. 
\citet{brock2017freezeout} builds a network in stages, where certain layers are progressively frozen, 
%which permits 
permitting
faster training.
\citet{mosca2017deep, wang2017deep} propose methods that progressively stack layers% and 
, performing end-to-end learning on the resulting network at each step.
A similar strategy was applied for training GANs in \citet{karras2017progressive}.
By the nature of our goals in this work, we never
%consider the 
perform
fine-tuning of the whole network. Several methods also consider auxiliary supervised objectives \cite{lee2015deeply} to stabilize end-to-end learning, but 
%which is rather 
this is 
different from the case where these objectives are not solved jointly.

\vspace*{-10pt}
\section{Supervised Layerwise Training of CNNs}\label{layerw}
\vspace*{-5pt}
In this section we formalize the architecture, training algorithm, and the necessary notations and terminology. We focus on CNNs, with ReLU non-linearity denoted by $\rho$. Sec. \ref{sec:arch} 
%will describe 
describes
a layerwise training scheme using a succession of auxiliary learning tasks. 
We add one layer at a time: 
the first layer of a $k$-hidden layer CNN problem. Finally, we %will 
discuss the distinctions in varying $k$. %trained upon the output of the last-trained layer. Finally, we will discuss the distinctions in varying $k$.}
%. that is similar to optimizing a $k$-hidden layer CNN, $k\geq 1$. The principle that we operate on is to add one layer at a time so that the layers objective can be well specified. Subsequently we will discuss the auxiliary problems considered in this work and their ramifications. 
\vspace*{-10pt}
\subsection{Architecture Formulation}\label{sec:arch}
\begin{figure*}
  \begin{minipage}{\textwidth}
  \begin{minipage}[b]{0.63\textwidth}
 %   \begin{figure}
       \begin{center}
     \includegraphics[width=\linewidth]{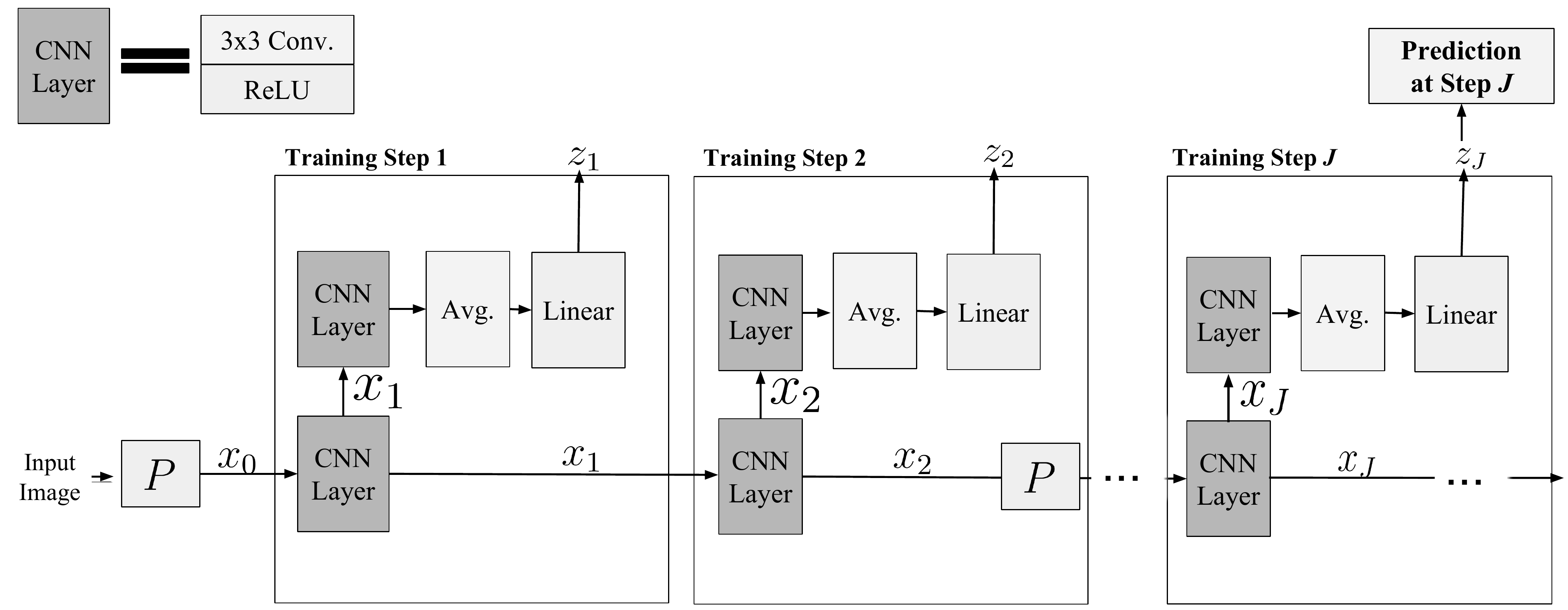}
     %\missingfigure[figwidth=9cm]{All mike charts}
       \end{center}
  %  \end{figure}
    \captionof{figure}{High level diagram of the layerwise CNN learning experimental framework using a $k=2$-hidden layer. $P$, the down-sampling (see Figure 2 \cite{jacobsen2018revnet}) , is applied at the input image as well as at $j=2$.\vspace{-15pt}}
    \label{fig:greedy}
  \end{minipage}
  \hfill
  \begin{minipage}[b]{0.35\textwidth}
  %  \centering
    \begin{algorithm}[H]
        \caption{Layer Wise CNN} 
        \label{algo}
        \begin{algorithmic}
        \small
       \STATE {\bfseries Input:} Training samples $\{x_0^n,y^n\}_{n\leq N}$
        \FOR{$j \in 0..J-1$}
            \STATE \texttt{Compute $\{x_j^n\}_{n\leq N}$ (via Eq.\eqref{eq:nonsymalgo})}\\
         %   \STATE \texttt{Initialize $\theta_j,\gamma_j$} \\
            \STATE $(\theta_j^*,\gamma_j^*)= \arg\min_{\theta_j,\gamma_j}\hat{\mathcal{R}}(z_{j+1};\theta_{j},\gamma_{j})$\\
        \ENDFOR
        \end{algorithmic}
        
    \end{algorithm}
    \vspace{0.55cm}
    %\\~\\~\\~
    \end{minipage}
   %\captionof{figure}{High level diagram of the layerwise CNN learning experimental framework using a $k=2$-hidden layer. $P$, the down-sampling (see Figure 2 \cite{jacobsen2018revnet}) , is applied at the input image as well as at $j=2$.\vspace{-15pt}}
   % \label{fig:greedy}
  \end{minipage}
\end{figure*}
\vspace*{-5pt}
Our architecture has $J$ blocks (see Fig. \ref{fig:greedy}), which are trained in succession.
From an input signal $x$, an initial representation $x_0 \triangleq x$ is propagated through $j$ convolutions, giving $x_j$.
Each $x_j$ feeds into an \textit{auxiliary classifier} to obtain prediction $z_j$, which computes an intermediate classification output. 
At depth $j$, denote by $W_{\theta_j}$ a convolutional operator with parameters $\theta_j$ , $C_{\gamma_j}$ an auxiliary classifier with all its parameters denoted $\gamma_j$, and $P_j$ a down-sampling operator. The parameters correspond to $3\times  3$ kernels with bias terms. Formally, from layer $x_j$ we iterate as follows:\begin{equation}
\left\{\begin{aligned}
	x_{j+1} &= \rho W_{\theta_j} P_jx_{j} \\
	z_{j+1}&= C_{\gamma_j} x_{j+1} \in \mathbb{R}^c
\end{aligned}
\right.\label{eq:nonsymalgo}\end{equation}where $c$ is the number of classes. 
For 
the pooling operator $P$ we choose the \textit{invertible downsampling} operation described in \citet{dinh2016density}, which consists in reorganizing the initial spatial channels into the 4 spatially decimated copies obtainable by $2\times 2$ spatial sub-sampling, reducing the resolution by a factor $2$.
We decided against strided pooling, average pooling, and the non-linear max-pooling, because these strongly encourage loss of information.
%reducing resolution size, we avoid non-linear pooling such as maxpooling and did not consider strides or average pooling  because they induce a loss of information, instead we employed the invertible downsampling operation, that we refer as $P$ described in  \citet{shi2016real,dinh2016density} that consists in a reshaping of the layer that maintains the spatial ordering while reducing the  resolution by a factor $2$ \cite[Fig.~2]{jacobsen2018revnet}. 
%At each layer $j$, the downsampling operator $P_j$ is applied at certain layers a
As is standard practice, $P$ is applied at certain layers ($P_j=P$), but not others ($P_j=Id$).
The CNN classifier $C_{\gamma_j}$ is given by:\begin{equation}C_{\gamma_j}x_j =\left\{
\begin{aligned}
	 LAx_j \quad \textrm{ for }  k = 1 \\
	LA\rho \tilde W_{k-2} ... \rho \tilde W_0 x_j  \quad \textrm{ for } k>1\\\end{aligned}\right.\end{equation}
 where $\tilde W_0, ..., \tilde W_{k-2}$ are convolutional layers with constant width, $A$ is a spatial averaging operator, and $L$ a linear operator whose output dimension is $c$. 
 %We remark the 
 The
 averaging operation is important for maintaining scalability at early layers.  
 %Observe that for 
 For
 $k=1$, $C_{\gamma_j}$ is %simply 
 a linear model. %, and in 
 In
 this case our architecture 
 %will be 
 is
 trained by a sequence of 1-hidden layer CNN.
\vspace*{-5pt}
\subsection{Training by Auxiliary Problems}
%
 %\begin{wrapfigure}{r}{0.45\textwidth}
%\begin{figure}[!t]
%\begin{minipage}[h]{0.5\textwidth}
%
%\end{minipage}
%\end{wrapfigure}
\vspace*{-5pt}
Our training procedure is layerwise: at depth $j$, while keeping all other parameters fixed, $\theta_j$ is obtained via an \textit{auxiliary problem}: optimizing $\{\theta_j,\gamma_j\}$ to obtain the best training accuracy for 
\textit{auxiliary classifier} $C_{\gamma_j}$.
%such that the auxiliary classifier $C_j$ obtains the best prediction accuracy on the training set.
We %now 
formalize this idea for a training set $\{x^n,y^n\}_{n\leq N}$:
For a function  $z(\cdot;\theta,\gamma)$  parametrized by $\{\theta,\gamma\}$ and a loss $l$  (e.g. cross entropy), we consider the classical minimization of the empirical risk:
 $\hat{\mathcal{R}}(z ; \theta,\gamma)\triangleq \frac 1 N \sum_n l(z(x^n; \theta,\gamma),y^n)$.

 At depth $j$, assume we have constructed the parameters $\{\hat\theta_0,...,\hat\theta_{j-1}\}$. %Interpreting $z_j$ as a function of $x$, our 
Our algorithm
%produced 
can produce
samples $\{x_j^n\}$.
Taking $z_{j+1}=z(x_j^n;\theta_j,\gamma_j)$,
we %will 
employ an optimization procedure that aims to minimize the risk $\hat{\mathcal{R}}(z_{j+1};\theta_j,\gamma_j)$. This procedure (Alg. 1) consists in training (e.g. using SGD) the shallow CNN classifier $C_j$ on top of $x_j$,  to obtain the new parameter $\hat\theta_{j+1}$. 
Under mild conditions, it improves the training error at each layer as shown below:
 % the empirical loss on a training set.
%We emphasize although the general problem is non-convex that for $k=1$ the possibility to build exact optimization exists \cite{}. Furthermore, initializing new layers with near identity is a common approach e.g. \cite{xolort_initialization} and thus with appropriate settings of optimization settings even standarad first order gradient methods should satisfy this condition.  
%In order to validate the benefits of training iteratively 2-hidden layer CNNs \mynotes{what benefits?}, we propose a variant of the above optimization procedure, by fixing the parameter $\gamma$ such that $W_\gamma=\mathbf{I}$. In this case, the whole procedure above is equivalent to train 1-hidden layer CNNs.\mynotes{what is this section what are u talking about here?}
\begin{prop}[Progressive improvement] \label{prop:improv} Assume that $P_j=Id$. Then there exists $\tilde\theta$ such that:
\vspace*{-5pt}
\[\hat{\mathcal{R}}(z_{j+1};\hat\theta_{j},\hat\gamma_{j})\leq\hat{\mathcal{R}}(z_{j+1};\tilde\theta,\hat\gamma_{j-1})=\hat{\mathcal{R}}(z_{j};\hat\theta_{j-1},\hat\gamma_{j-1}).\]
\end{prop}
%\begin{proof}
%As $\rho(\rho( x)) = \rho(x)$, we simply have to chose $\theta_0$ such that $W_{\theta_0}=Id$.
%At depth $j$, observe that in the setting of the paper, we can find a parameter $\theta_j$ such that: $\rho W_{\theta_j}=I$ and if we set $C_{\gamma_{j+1}}=C_{\gamma_j^*}$, then the loss at depth $j+1$ is equal to: $\mathcal{R}_j^*$ and thus $\mathcal{R}^*_{j+1}\leq \mathcal{R}^*_{j}$.
%\end{proof}
%This property is achievable in our case because of the ReLU. 
\vspace*{-5pt}
A technical requirement for the actual optimization procedure is to not produce a worse objective than the initialization, which can be achieved by  taking the best result along the optimization trajectory.

The cascade can inherit
from the individual properties of each auxiliary problem. For instance, as $\rho$ is 1-Lipschitz, if each $W_{\hat\theta_j}$ is 1-Lipschitz then so is $x_J$ w.r.t. $x$.   Another example is the nested objective defined by Alg. 1: the optimality of the solution will be largely governed by the optimality of the sub-problem solver. Specifically, if the auxiliary problem solution is close to optimal %than 
then
the solution of Alg. 1 will be close to optimal.

\begin{prop}
\label{prop:opt}Assume the parameters $\{\theta_{0}^*,...,\theta_{J-1}^*\}$ are obtained via a optimal layerwise optimization procedure. We assume that $ W_{\theta_j^*}$ is 1-lipschitz without loss of generality and that the biases are bounded uniformly by $B$. Given an input function $g(x)$, we consider functions of the type $z_{g}(x)=C_\gamma \rho W_\theta g(x)$. For $\epsilon>0$, we call $\theta_{\epsilon,g}$ the parameter provided by a procedure to minimize $\hat{\mathcal{R}}(z_{g};\theta;\gamma)$ which leads to a 1-lipschitz operator that satisfies:
1.~$ \underbrace{\Vert \rho W_{\theta_{\epsilon,g}}g(x) -\rho W_{\theta_{\epsilon,\tilde g}}\tilde{g}(x)\Vert\leq\Vert g(x)-\tilde{g}(x)\Vert}_\text{ (stability)},\forall g,\tilde g,\quad$ \\ 2.~~~~~~~$\underbrace{\Vert W_{\theta_j^*} x_j^*-W_{\theta_{\epsilon,x_j^*}} x_j^*\Vert\leq\epsilon(1+\Vert  x_j^* \Vert)}_\text{ ($\epsilon$-approximation)},$

with, $\hat x_{j+1}=\rho W_{\theta_{\epsilon,\hat x_j}}\hat x_j$ and $ x_{j+1}^*=\rho W_{\theta_j^*} x_j^*$ with $x_0^*=\hat x_0=x$, then, we prove by induction:\begin{equation}
\Vert  x_J^*-\hat x_{J}\Vert=\mathcal{O}(J^2\epsilon) \end{equation}\end{prop}
\vspace{-5pt}
The proof can be found in the Appendix \ref{appendix:proof_prop_opt}. This demonstrates an example of how the training strategy can permit to extend results from shallow CNNs to deeper CNNs, in particular for $k=1$. Applying an existing optimization strategy could give us a bound on the solution of the overall objective of Alg. 1, as we will discuss below.
%\textcolor{red}{Help to clarify this}%Indeed, this setting is well studied in the literature, as we will explain in the next subsection.

\vspace*{-5pt}
\subsection{Auxiliary Problems \& The Properties They Induce}\label{sec:aux_prop}
\vspace*{-5pt}
We now discuss the properties arising from the auxiliary problems. We start with $k=1$, for which the auxiliary classifier consists of only the linear $A$ and $L$ operators. Thus, the optimization aims to obtain the weights of a 1-hidden layer NN. For this case, as discussed in Sec.~\ref{sec:intro}, a variety of theoretical results exist (e.g. \citep{cybenko1989approximation, barron1994approximation}).  Moreover, ~\cite{AroraBMM17,ge2017learning,1cnntrain,bach2014breaking} proposed provable optimization strategies for this case. Thus the analysis and optimization of the 1-hidden layer problem is a case that is relatively well understood compared to deep counterparts. At the same time, as shown in Prop. \ref{prop:opt}, applying an existing optimization strategy could give us a bound on the solution of the overall objective of Alg. 1. 
To build intuition let us consider another example  where the analysis can be simplified for $k=1$ training. Recall  the  classic least square estimator  \cite{barron1994approximation} of a 1-hidden layer network: %that involve the empirical estimator of the minimizer of the mean squared integrated error, over a r.v. $X$:
%\vspace{-5pt}
\begin{equation}(\hat L,\hat \theta)=\arg\inf_{(L,\theta)}\sum_n \Vert f(x^n)-L\rho W_{\theta}x^n \Vert^2\label{eq-mse}\end{equation}
where  $f$ is the function of interest. Following a suggestion from \cite{mallat2016understanding} (detailed in Appendix A) we can state there exists a set $\Omega_j$ and $ f_j$, $ \forall x\in \Omega_j, f(x)= f_j\circ\rho W_{\hat\theta_{j-1}}...\rho W_{\hat\theta_1}x$ where $\{\hat\theta_{j-1},...,\hat\theta_1\}$ are the parameters of greedily trained layers (with width  $\alpha_j$) and $\rho$ is sigmoidal. For simplicity, let us assume that $x^n\in\Omega_j,\forall n$. It implies that at step 
%each 
$j$ of a greedy training procedure with $k=1$, the corresponding sample loss is:
\[\Vert f(x^n)-z(x_j^n;\theta_j,\gamma_j)\Vert^2 = \Vert  f_j(x_j^n)-z(x_j^n;\theta_j,\gamma_j)\Vert^2\,.\]
%\begin{align*}
%\eta_j=&\inf_{\theta_{j},\gamma_j}\mathbb{E}_{X_0}\big[\Vert f(X_0)-z(X_j;\theta_j,\gamma_j) \Vert^2\big]\\=&
%\inf_{\theta_{j},\gamma_j}\mathbb{E}_{X_j}\big[\Vert \hat f_j(X_j)-z(X_j;\theta_j,\gamma_j) \Vert^2\big]
%\end{align*}
 In particular, the right term is shaped as Eq. \eqref{eq-mse} and thus we can apply standard bounds available only for 1-hidden layer settings \cite{barron1994approximation,janzamin2015beating}%
%on it. 
.
In the case of jointly learning the layers we could not make this kind of formulation. For example if one now applied the algorithm of \cite{janzamin2015beating} and  their Theorem 5 it would give the following risk bound:
\begin{align*}
    \mathbb{E}_{X_j}[\Vert f_j(X_j)-z(X_j;\hat\theta_j,\hat\gamma_j)\Vert^2]\leq  \mathcal{O}&\bigg(C_ {f_j}^2(\frac{1}{\sqrt{\alpha_j}}+\delta_\rho)^2\bigg)\\
    &+\mathcal{O}(\epsilon^2)\,,
\end{align*}
where $X_j=\rho W_{\hat \theta_{j-1}}...\rho W_{\hat \theta_1}X_0$ is obtained from an initial bounded distribution $X_0$,  $C_{f_j}$ is the Barron Constant (as described in \cite{lee2017ability}) of $f_j$, $\delta_\rho$ a constant depending on $\rho$ and $\epsilon$ an estimation error. Furthermore, if we fit the $f_j$ layerwise using a method such as  \cite{janzamin2015beating}, %one can reach an approximate solution for the optimization of the cascade. 
we reach an approximate optimum for each layer given the state of the previous layer.
Under Prop~\ref{prop:opt}, we observe that small errors
at each layer, even taken cumulatively,
will not affect the final representation learned by the cascade. Observe that if $C_{f_j}$ decreases with $j$, the approximation bound on $f_j$ will correspondingly improve.

\mynotes{Another view of the $k=1$ training using a standard classification objective: the optimization of the 1-hidden layer network will encourage the hidden layer outputs to make its classification output maximally linearly separable with respect to its inputs.  Specializing Prop. \ref{prop:improv} for this case shows that the layerwise $k=1$ procedure will try to progressively improve the linear separability. Progressive linear separation has been empirically studied in end-to-end CNNs \citep{zeiler2014visualizing,oyallon2017building} as an indirect consequence, while the $k=1$ training permits us to study this basic principle more directly as the layer objective. % in the sequel.
%We note also concurrent 
Concurrent
work \cite{elad2019the} follows
the same argument to use a layerwise training procedure to evaluate mutual information more directly.}

%\paragraph{$k>1$ auxiliary problems} 
Unique to our layerwise learning formulation, we consider the case where the auxiliary learning problem involves auxiliary hidden layers. We will interpret, and empirically verify, in Sec.  \ref{sec:prog_exp} that this builds layers that are progressively better inputs to shallow CNNs. We will also show a link to building, in a more progressive manner, linearly separable layers. Considering only shallow (with respect to total depth) auxiliary problems (e.g. $k=2,3$ in our work) we can maintain several advantages.  Indeed, optimization for shallower networks is generally easier, as we can for example diminish the vanishing gradient problem, reducing the need for identity loops or normalization techniques \citep{he2016deep}. Two and three hidden layer networks are also appealing for extending results from one hidden layer \citep{threelayer} as they are the next natural member in the family of NNs.
%\mynotes{ while bringing potentially more expressivity power \citep{eldan2016power}}. 

%However,  \cite{bach2014breaking}, explains multiple layer neural networks seem to bear fundamentally different principles from one layer.  Two layers CNNs are appealing for extending results from one hidden layer to two. First, with appropriate non-linearities, 2-hidden layer neural networks can be view as a special case of a $k$-hidden layer neural network with identity layers \mynotes{Tautology?}, which makes them the simplest member in the family of deep networks that go beyond one layer, a family that has the least existing theoretical results. m \mynotes{most complex is 1+?}.

%In a supervised context, the first and final layer are simpler to interpret because they inherit from the structures of the its data extremety \mynotes{not clear}: in the case of an image classification task, the channel indexes of the inputs are structured by the translation whereas the output of the CNN is expected to be linearly separable.
\vspace*{-5pt}
\section{Experiments and Discussion}\label{sec:exp}
%We now describe our experimental results that utilize 
We performed experiments on
the large-scale ImageNet-1k \citep{ILSVRC15}, a major catalyst for the recent popularity of deep learning, as well as the CIFAR-10 dataset. We study the classification performance of layerwise models with $k=1$, comparing them to standard benchmarks and other sequential learning methods. Then we  inspect the representations built through our auxiliary tasks and motivate the use of models learned with auxiliary hidden layers $k>1$, which we subsequently evaluate at scale. 

We highlight that many algorithms do not scale to large datasets \cite{bartunov2018assessing} like ImageNet. For example a SOTA hand-crafted image descriptor combined with a linear model achieves $82\%$ on CIFAR-10 while only $17.4\%$ on ImageNet \cite{PAMIScat}. 1-hidden layer CNNs from our experiments can obtain $61\%$ accuracy on CIFAR-10 while the same CNN results on ImageNet only gives $12.3\%$ accuracy. %Indeed, alternative
Alternative
learning methods for deep networks can sometimes show similar behavior, highlighting the importance of assessing their scalability. For example Feedback Alignment \cite{lillicrap2016random} which is able to achieve $63\%$ on CIFAR-10 compared to $68\%$ with backprop, on the 1000 class ImageNet obtains only $6.6\%$ accuracy compared to $50.9\%$ \cite{bartunov2018assessing,Xiao2019}.
Thus, based on these observations and the sparse and small-scale experimental %work
efforts
of related works on greedy layerwise learning it is entirely unclear whether this family of approaches can work on large datasets like ImageNet and be used to construct useful models. Noting that ImageNet models not only represent benchmarks but have generic features \cite{yosinski2014transferable} and thus AlexNet- and VGG-like accuracies for CNN's on this dataset typically indicate representations that are generic enough to be useful for downstream applications.
\vspace*{-10pt}
\paragraph{Naming Conventions} We call $M$ the number of feature maps of the first convolution of the network and $\tilde M$ the number of feature maps of the first convolution of the auxiliary classifiers.
This fully defines the width of all the layers, since input width and output width are equal unless the layer has downsampling, in which case the output width is twice the input width.
%layer: the output width is equal to the input width, except when a downsampling is incorporated at depth $j$.
%In this case, the larger resulting size is compensated by a reduction by 2 of the number of channels of $W_{\gamma_j}$ as well as the width of each layer of $C_{\gamma_j}$, for computational issues. 
Finally,  $A$ is chosen to average over the four spatial quadrants, yielding a $2\times 2$-shaped output.
Spatial averaging before the linear layer is common in ResNets \citep{he2016deep} to reduce size.
In our case this is critical to permit scalability to large image sizes at early layers of layer-wise training. For computational reasons on ImageNet, an invertible downsampling is also applied (reducing the signal to output $12\times112^2$). We also construct an ensemble model, which consists of  a weighted average of all auxiliary classifier outputs, i.e. $Z = \sum_{j=1}^J {2^{j}} z_j$. \mynotes{Our sandbox architectures to study layerwise learning end with a final auxiliary network that becomes part of the final model. %Indeed in
In
some cases we may want to 
%thus 
use a different final auxiliary network
after training the layer%
. We will use $\tilde{M}_f$ to denote the width of the final auxiliary CNN. For shorthand we will denote our simple CNN networks SimCNN.}

We briefly introduce the datasets and preprocessing.
CIFAR-10 consists of small RGB images with respectively $50k$ and $10k$ samples for training and testing.  
We use the standard data augmentation and optimize each layer with SGD using a momentum of 0.9 and a batch-size of 128.
The initial learning rate is  $0.1$ and we use the reduced schedule with decays of $0.2$ every $15$ epochs~\citep{zagoruyko2016wide}, for a total of 50 epochs in each layer. 
ImageNet  consists of $1.2M$ RGB images of  varying size for training.
Our data augmentation consists of random crops of size $224^2$.
At testing time, the image is rescaled to $256^2$ then cropped at size $224^2$. 
We used SGD with momentum 0.9 for a batch size of 256. 
The initial learning rate is $0.1$~\citep{he2016deep} and we use the reduced schedule with decays of $0.1$ every 20 epochs for 45 epochs. 
We use 4 GPUs to train our ImageNet models. 
%The final networks  are shown in further details in the Appendix\ref{layerw}.. % We use a batch size of 256 and we apply a reduced training schedule that consists of 45 epochs and learning rate drops every 20 epochs.
\vspace{-5pt}
\subsection{AlexNet Accuracy with 1-Hidden Layer Auxiliary Problems}\label{sec:alex}
\vspace{-5pt}
We consider the atomic, layerwise CNN with $k=1$ which corresponds to solving a sequence of 1-hidden layer CNN problems. 
As discussed in Sec. \ref{sec:rel}, previous attempts at supervised layerwise training \citep{NIPS1989_207,arora2014provable,huang2017learning,malach2018provably}, which rely solely on sequential solving of shallow problems have yielded performance well below that of typical deep learning models even on the CIFAR-10 dataset. 
We show, surprisingly, that it is possible to go beyond the AlexNet performance barrier \citep{krizhevsky2012imagenet} without end-to-end backpropagation on ImageNet with  elementary auxiliary problems. To emphasize the stability of the training process we do not apply any batch-normalization to this model.
\vspace{-8pt}
\paragraph{CIFAR-10.} 
We trained a model with $J=5$ layers, down-sampling at layers $j=2,3$, and layer sizes starting at $M=256$. 
We obtain $88.3\%$ and note that this accuracy is close to the AlexNet model performance~\citep{krizhevsky2012imagenet} for CIFAR-10 (89.0\%). Comparisons in Table \ref{CIFAR} show that other 
sequentially trained  1-hidden layer networks have yielded performances that do not exceed those of the top hand-crafted methods or those using unsupervised learning. To the best of our knowledge they obtain $82.0\%$ accuracy~\citep{huang2017learning}. The end-to-end version of this model obtains $89.7\%$, while using $5\times$ more GPU memory than $k=1$ training. 

\paragraph{ImageNet.} 
\begin{figure}
    \centering
    \includegraphics[scale=0.36]{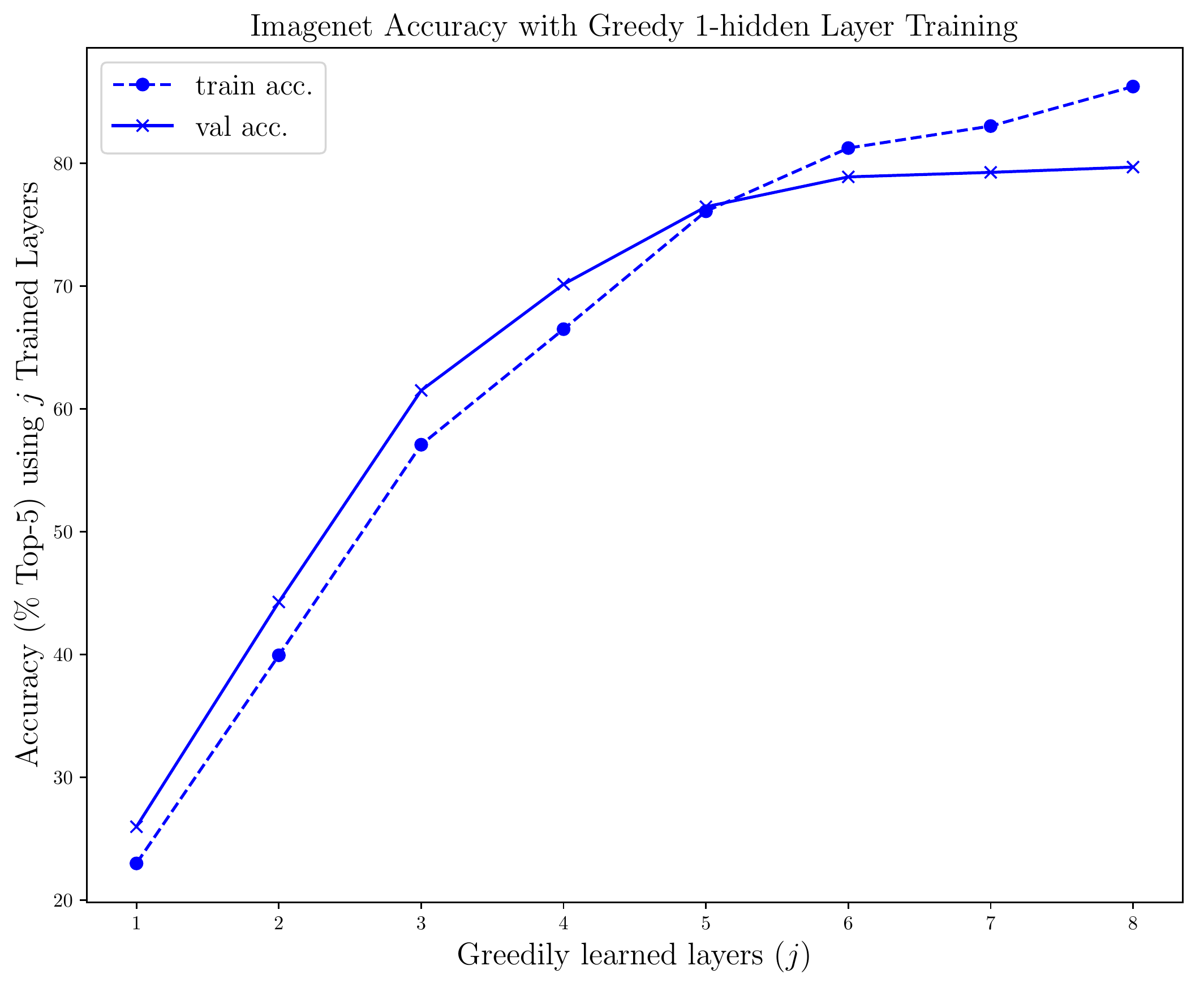}
    \caption{One-hidden layer trained ImageNet model shows rapid progressive improvement\vspace{-20pt}}
    \label{fig:prog}
\end{figure}
Our model is trained with $J=8$ layers and downsampling operations at layers $j=2,3,4,6$.
Layer sizes start at $M=256$.  
Our final trained model achieves \textbf{79.7\%} top-5 single crop accuracy on the validation set and \textbf{80.8\%} a weighted ensemble of the layer outputs. 
In addition to exceeding AlexNet, this model compares favorably to alternatives to end-to-end supervised CNNs including hand-crafted computer vision techniques \citep{sanchez2013image}. Full results are shown in Table \ref{tab:imagenet_full} where we also highlight %huge 
the substantial
performance gaps %of
to
multiple alternative training methods. Figure~\ref{fig:prog}  shows the per-layer classification  accuracy. We  observe a remarkably rapid progression (near linear in the first 5  layers) that takes the accuracy from $\sim 23\%$ to $\sim 80\%$ top-5. We leave a full theoretical explanation of this fast rate as an open question. Finally, in Appendix \ref{app:transfer} we demonstrate 
that
this model maintains the %famous 
transfer learning properties of deep networks trained on ImageNet.  

We also note that our final training accuracy is relatively high for ImageNet (87\%), which indicates that appropriate regularization may lead to a further improvement in test accuracy. We now look at empirical properties induced in the layers and subsequently evaluate the distinct $k>1$.
\vspace*{-8pt}
\subsection{Empirical Separability Properties}\label{sec:prog_exp}
We study the intermediate representations generated by the layerwise learning procedure in terms of linear separability as well as separability by a more general set of classifiers.  Our aims are 
\textbf{(a)} to determine empirically whether $k=1$ indeed progressively builds more and more linearly separable data representations and 
\textbf{(b)} to determine how linear separability of the representations evolves for networks constructed with $k>1$ auxiliary problems. 
Finally we ask whether the notion of building progressively better inputs to a linear model ($k=1$ training) has an analogous counterpart for $k > 1$:
building progressively better inputs for shallow CNNs (discussed in Sec~\ref{sec:aux_prop}).
%We note this is one interpretation of the $k>1$ sequential training. 

We define \textit{linear separability} of a representation as the maximum accuracy achievable by a linear classifier. Further we define the notion of \textit{CNN-}p\textit{-separability} as the  accuracy achieved by a $p$-layer CNN trained on top of the representation to be assessed.
%\textit{test set}

We focus on CNNs trained on CIFAR-10 without downsampling. Here, $J=5$ and the layer sizes follow $M=64,128,256$. 
The auxiliary classifier feature map size, when applicable, is $\tilde M = 256$. 
We train with 5 random initializations for each network and report an average standard deviation of 0.58\% test accuracy. 
Each layer is evaluated by training a one-versus-rest logistic regression, as well as $p=1,2$-hidden-layer CNN on top of these representations. 
Because the linear representation has been optimized for it, we spatially average to a $2\times 2$ shape before feeding them to the separability evaluation. 
Fig. \ref{fig:cifar_exp_linear} shows the results of each of these evaluations plotting test set accuracy curves as a function of NN depth for each of the 3 evaluations. 
For these plots we averaged over initial layer sizes $M$ and classifier layer sizes $\tilde M$ and random seeds. Each individual curve closely resembles these average curves, with slight shifts in the y-axis.

As expected  from Sec. \ref{sec:aux_prop}, linear separability monotonically increases with  depth  for $k=1$. Interestingly,  linear separability also improves in the case of $k>1$, even though it is not directly specified by the auxiliary problem objective. 
At earlier layers, linear separation capability of models trained with $k=1$ increases fastest as a function of layer depth compared to models trained with deeper auxiliary networks, but flattens out to a lower asymptotic linear separability at deeper layers. Thus, the simple principle of the $k=1$ objective that tries to produce the maximal linear separation at each layer might not be an optimal strategy for achieving progressive linear separation. 

We also notice that the deeper the auxiliary classifier, the slower the increase in linear separability initially, but the higher the linear separability at deeper layers. From the two right diagrams we also find that the CNN-$p$-separability progressively improves - but much more so for $k>1$ trained networks. 
This shows that linear separability of a layer is not the sole criterion for rendering a representation a good "input" for a CNN.
It further shows that our sequential training procedure for the case $k>1$ can  build a representation that is progressively a better input to a shallow CNN. %

\begin{figure*}
  \begin{minipage}{\textwidth}
  \begin{minipage}[b]{0.73\textwidth}
 %   \begin{figure}
       \begin{center}
      \includegraphics[width=\linewidth]{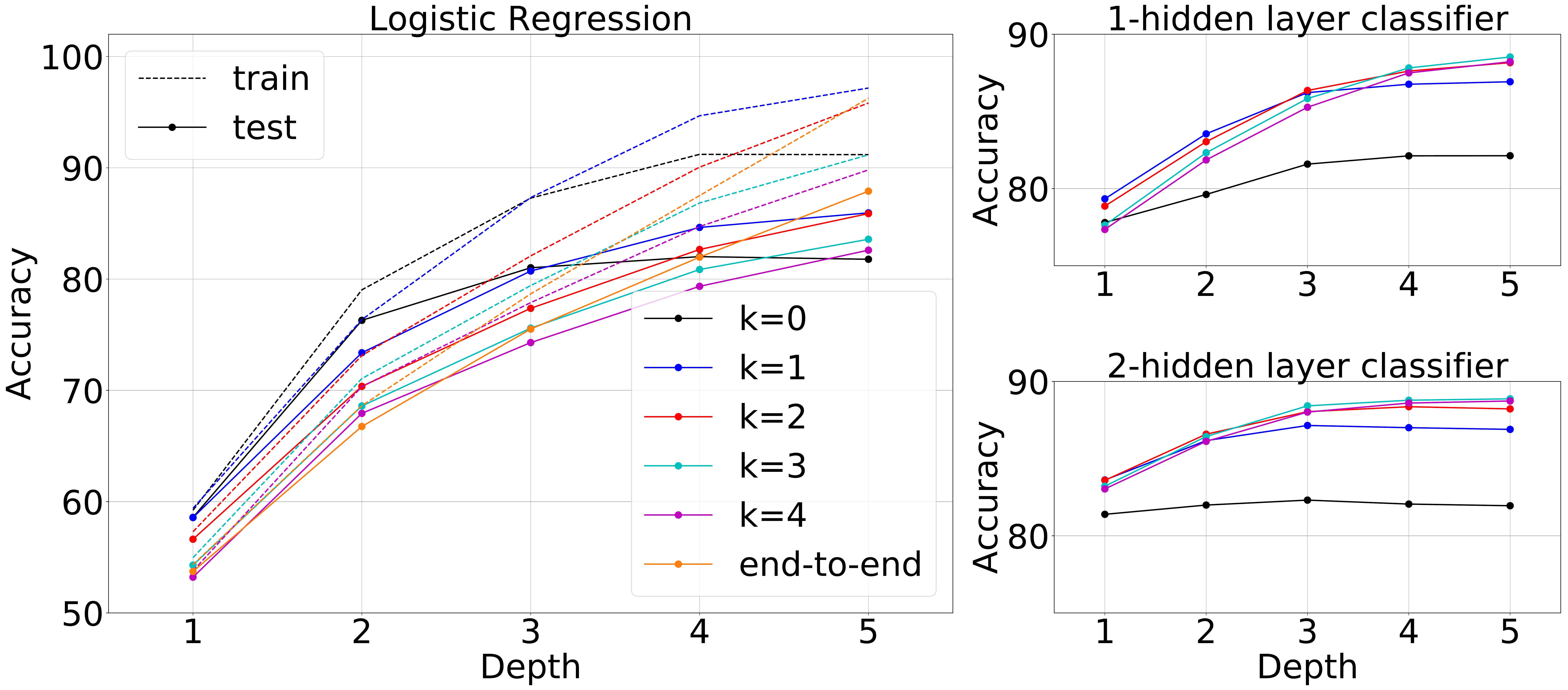}\vspace{-0.5cm}
     %\missingfigure[figwidth=9cm]{All mike charts}
       \end{center}

  \end{minipage}
  \hfill
  \begin{minipage}[b]{0.24\textwidth}
    \centering
    \begin{tabular}{|c|c|}
     \hline
    \textbf{\makecell{}}  & \makecell{Top-5} \\\hline
        \makecell{$\tilde{M}_f=2048$  }    & \makecell{ 88.7 }\\\hline
        \makecell{$\tilde{M}_f=512$  }& 88.5\\\hline
      % \makecell{$\tilde{M}_f=512$  (e2e train)}&  90.1\\\hline
      %  VGG-11 ($k=3$ train) &  68.1 & 88.1\\\hline
     %   VGG-11 (e2e train) & \makecell{68.1 \\ (70.1)} &	\makecell{88.1 \\ (89.3)}\\\hline
     %   VGG-11 ($k=3$ train) & 70.4 &	89.8 \\\hline
     \end{tabular}\captionof{table}{We compare ImageNet SimCNN ($k=3$) with a much smaller final auxiliary model. The final accuracy is diminished only slightly while the model becomes drastically smaller.\vspace{0.2cm} \label{tab:width_aux}}
    \end{minipage} \captionof{figure}{(Left) Linear and (Right) CNN-$p$ separability as a function of depth for CIFAR-10 models. For Linear separability we aggregate across $M=64,128,256$, individual results are shown in Appendix ~\ref{app:cifar_exp}, the relative trends are largely unchanged. For CNN-$p$ probes, all models achieve $100\%$ train accuracy at the first or 2nd layer, thus only test accuracy is reported.\vspace{-15pt}\label{fig:cifar_exp_linear}}
  \end{minipage}
  
\end{figure*}
\vspace*{-8pt}
\subsection{Scaling up Layerwise CNNs with 2 and 3 Hidden Layer Auxiliary Problems}\label{sec:scaling}

%\mynotes{Based on the previous section we study training of deep networks with 2 and 3-hidden layer sub-problems.} 
\begin{table}
\centering
\begin{tabular}{|c|c|}
\hline
\textbf{\makecell{Layer-wise Trained}}& \makecell{  Acc.  (Ens.)}\\\hline
    SimCNN ($k=1$ train )& \makecell{  88.3  (88.4)} \\\hline
    SimCNN ($k=2$ train) &   \makecell{ 90.4 (90.7)}\\\hline
    SimCNN($k=3$ train) &  \makecell{ 91.7  (\textbf{92.8})}\\\hline
    BoostResnet   \citep{huang2017learning} & 82.1 \\\hline
    ProvableNN (Malach et al., 2018)& 73.4  \\\hline
   (Mosca et al., 2017)
   & 81.6\\\hline
%   \mynotes{ SimCNN (e2e train )} & \makecell{  89.7} \\\hline
     
     \textbf{Reference e2e}& \\\hline
    \makecell{  AlexNet } & 89 \\\hline
     VGG $^1$ & 92.5 \\\hline
    %\makecell{  All-CNN \\ \citep{springenberg2014striving}}& 92.8\\\hline  
    WRN 28-10 (Zagoruyko et al. 2016)
    %\citep{zagoruyko2016wide}
    & \textbf{96.0} \\
    \hline
    
    \hline
    \textbf{Alternatives} & [Ref.] \\\hline
     \makecell{Scattering + Linear \\} & 82.3 \\\hline
     \makecell{ FeedbackAlign \cite{bartunov2018assessing}
     }& 62.6 [67.6]\\
     \hline\end{tabular}\caption{Results on CIFAR-10. Compared to the few existing methods using \textit{only} layerwise training schemes we report much more competitive results to well known benchmark models that like ours do not use skip connnections.In brackets e2e trained version of the model is shown when available.
     \label{CIFAR}\vspace{-30pt}\\}
     \end{table}
We study the training of deep networks with $k=2,3$ hidden layer auxiliary problems. 
We limit ourselves to this setting to keep the auxiliary NNs shallow with respect to the network depth.
We employ widths of $M=128$ and $\tilde{M}=256$ for both CIFAR-10 and ImageNet. 
For CIFAR-10, we use $J=4$ layers and a down-sampling  at  $j=2$. 
For ImageNet we closely follow the VGG architectures, which with their $3\times 3$ convolutions and absence of skip-connections bear strong similarity to ours.
We use $J=8$ layers. 
%(there are $J + k$ layers in total including the final linear projection - see also Fig \ref{fig:greedy}).
As we start at halved resolution we do only 3 down-samplings at $j=2,4,6$. 
Unlike the $k=1$ case we found it helpful to employ batch-norm for these auxiliary problems. 

\begin{table}[t]
\centering
\begin{tabular}{|c|c|c|}
 \hline
\textbf{\makecell{}}  &  \makecell{Top-1 (Ens.)}  &\makecell{Top-5 (Ens.)} \\\hline
     
    SimCNN ($k=1$ train) & \makecell{  58.1 (59.3)}  &  \makecell{ 79.7 (80.8)}   \\\hline
    SimCNN ($k=2$ train)   &  \makecell{ 65.7 (67.1) }&  \makecell{ 86.3 (87.0) }  \\\hline
    SimCNN ($k=3$ train)    & \makecell{  69.7  (71.6)} &  \makecell{ 88.7 (89.8)}\\\hline
  %  \makecell{SimCNN $L,\tilde{M}_f=2,512$ \\ ($k=3$ train)}&68.7&88.5\\\hline
    %\makecell{SimCNN $L,\tilde{M}_f=2,512$ \\  (e2e train)}& 71.5 & 90.1\\\hline
    VGG-11 ($k=3$ train) & \makecell{ 67.6  (70.1)} & \makecell{88.0 (89.2)} \\\hline
    VGG-11 (e2e train) &  67.9  & 88.0 \\\hline\hline
 %   VGG-11 ($k=3$ train) & 70.4 &	89.8 \\\hline

    \textbf{Alternative} & [Ref.] & [Ref.]\\\hline
      \makecell{ DTargetProp  \\
     \cite{bartunov2018assessing}
      }&  \makecell{1.6 $[28.6]$}& \makecell{5.4 $[51.0]$} \\\hline
      \makecell{ FeedbackAlign \\ \cite{Xiao2019}
     %\citep{noroozi2016unsupervised}
      }&  \makecell{ 6.6  $[50.9]$}& \makecell{16.7  $[75.0]$} \\\hline
 %   \makecell{  FV+ MLP (Perronin et al., 2015)\\ %\citep{perronnin2015fisher}
  %  }& \textbf{55.6} & \textbf{78.4}\\\hline
    \makecell{ Scat. + Linear \\ \citep{PAMIScat}
     }& 17.4 & N/A \\\hline
     \makecell{ Random CNN \\ 
     }& 12.9 & N/A \\\hline
      \makecell{ FV + Linear \\ \citep{sanchez2013image}
     }& 54.3 & 74.3 \\\hline
     \textbf{Reference e2e CNN}  &  &\\\hline%\bottomrule
        AlexNet  & 56.5 & 79.1\\\hline
    VGG-13  & 69.9 & 89.3 \\\hline
    VGG-19  & 72.9 & 90.9 \\\hline
    Resnet-152  & 78.3 & 94.1 \\\hline
    % end-to-end of $k=3$,$\tilde{M}=1024$  & 72.3 & 90.6\\
        % \hline
   %  \makecell{ Scat + SLE \citep{oyallon2017scaling}}& \textbf{57.0} & \textbf{79.6}\\\hline
 
 \end{tabular}\caption{\label{tab:imagenet_full}
 Single crop validation acc. on ImageNet. Our SimCNN models use $J=8$. In parentheses see the ensemble prediction. Layer-wise models are competitive with well known ImageNet benchmarks that similarly don't use skip connections. $k=3$ training can yield equal performance to end to end on VGG-11. We highlight many methods and alternative training do not work at all on ImageNet. In brackets, e2e acc.  is shown when available.\vspace{-19pt}}\end{table}
 
 We report our results for $k=2,3$ in Table \ref{CIFAR} (CIFAR-10) and  Table~\ref{tab:imagenet_full} (ImageNet) along with our $k=1$ model. The reference model accuracies for ImageNet AlexNet, VGG, and ResNet use the same input sizes and single crop evaluation\footnote{\small
Accuracies  from  http://torch.ch/blog/2015/07/30/cifar.html and https://pytorch.org/docs/master/torchvision/models.html.}.
As expected from the previous section, the transition from $k=1$ to $k=2,3$ improves the performances substantially.
We compare our CIFAR-10 results to other sequentially trained propositions in the literature.
Our methods exceed these in performance by a large margin.
While the ensemble model of $k=3$ surpasses the VGG, the other sequential models perform do not exceed unsupervised methods. %\cite{oyallon2015deep, dosovitskiy2014discriminative}.
No alternative sequential models are available for ImageNet.  
We thus compare our results on ImageNet to the standard reference CNNs and the best-performing alternatives to end-to-end Deep CNNs. Our $k=3$ layerwise ensemble model achieves \textbf{89.8\%} top-5 accuracy, which is comparable to VGG-13 and largely exceeds AlexNet performance. Recall that $\tilde{M}_f$ denotes the width of the final auxiliary network, which becomes part of the model. In the experiments above this is relatively large  ($\tilde{M}_f=2048$), while the final layers have diminishing returns.  
%We thus now consider to 
To
more effectively incorporate the final auxiliary network into the model architecutre, we reduce the width of the final \textit{auxiliary} network for the $k=3$ recomputing the $j=7$ step. We use auxiliary networks of size $\tilde{M}_f=512$ (instead of $2048$). In Table \ref{tab:width_aux} we observe that while the model size is reduced substantially, 
%we observe 
there is
only a limited loss of accuracy. %We can thus observe 
Thus,
the final auxiliary model structure does not heavily affect performance, suggesting we can train more generic architectures with this approach.
%For comparison, we train an end-to-end network with the same architecture as our $J=8$ network with the final auxiliary of $\tilde{M_f}=512$. Tab.~\ref{tab:imagenet_full} shows the accuracy of the end-to-end model is close $90.1\%$ top-5 compared to $88.5\%$ top-5 for the sequentially trained. 
 
\paragraph{Comparison to end-to-end VGG} We now compare this training method to end-to-end learning directly on a common model, VGG-11. We use $k=3$ training for all but the final convolutional layer.
We additionally reduce the spatial resolution by $2\times$ before applying the auxiliary convolutions. As in the last experiment we use a different final auxiliary network so that the model architecture matches VGG-11. The last convolutional layers' auxiliary model simply becomes the final max-pooling and 2-hidden layer fully connected network of VGG. We train a baseline VGG-11 model using the same 45 epoch training schedule.  The performance of the $k=3$ training strategy matches that of the end-to-end training, with the ensemble model being better. The main differences of this model to SimCNN, besides the final auxiliary, is the max-pooling and the input image starting from full size. %\mynotes{Indeed we have observed that the strategy works better with larger input images}  
%\mynotes{In general we observed that the layerwise CNN learning works better with larger starting resolution of the images.}
 
%\mynotes{We show a direction to close this relatively small gap in the next section.}

%We also compare to non-deep learning based methods and finally to several standard benchmarks.
Reference models relying on residual connections and very deep networks have better performance than those considered here. We believe that one can extend layer-wise learning to these modern techniques. However, this is outside the scope of this work. Moreover, recent ImageNet models (after VGG) are developed in industry settings, with large-scale infrastructure available for architecture and hyper-parameter search. Better design of sub-problem optimization geared for this setting may further improve results.
%while our models are developed on modest academic resources, which is why we restrict our-self simple architectural elements.   
%Our models were developed on smaller academic resources.

We emphasize that this approach enables the training of larger layer-wise models than end-to-end ones on the same hardware.
This suggests applications in fields with large models (e.g. 3-D vision and medical imaging). 
We also observed that using outputs of early layers that were not yet converged still permitted improvement in subsequent layers.
This suggests that our work might allow an extension that solves the auxiliary problems in parallel to a certain degree.

\paragraph{Layerwise Model Compression}
Wide, overparametrized, layers have been shown to be important for learning \citep{neyshabur2018towards}, but it is often possible to reduce the layer size \textit{a posteriori} without losing significant accuracy  ~\citep{hinton2014dark,lecun1990optimal}.
For the specific case of CNNs, one technique removes channels heuristically and then fine-tunes \citep{molchanov2016pruning}. In our setting, a natural strategy presents itself, which integrates compression into the learning process:
\textbf{(a)} train a new layer (via an auxiliary problem) and 
\textbf{(b)} immediately apply model compression to the new layer.
The model-compression-related fine-tuning operates over a single layer, making it fast and the subsequent training steps have a smaller input and thus fewer parameters, which speeds up the sequential training.
We implement this approach using the filter removal technique of \citet{molchanov2016pruning} only at each newly trained layer, followed by a fine-tuning of the auxiliary network.
We test this idea on CIFAR-10. A baseline network of 5 layers of size $64$ (no downsampling, trained for 120 epochs and lr drops each 25 epochs) obtains an end-to-end performance of $87.5\%$.
We use our layer-wise learning with $k=3,J=3,M=128,\tilde{M}=128$.
At each step we prune each layer from $128$ to $64$ filters and subsequently fine-tune \textit{the auxiliary network} to the remaining features over 20 epochs.
We then use a final auxiliary of $\tilde{M}_f=64$ obtaining a sequentially learned, final network of the same architecture as the baseline.
The final accuracy is $87.6\%$, which is very close to the baseline.
We note that each auxiliary problem incurs minimal reduction in accuracy through feature reduction.
%\mynotes{Unlike the previous experiment, where our final performance was slightly below that of end-to-end on the same architecture, this gap could be closed by easy-to-integrate compression approaches. }

%\subsection{Discussion}\label{benefits}

\vspace{-10pt}
\section{Conclusion}\label{sec:discussion}
\vspace{-5pt}
We have shown that an alternative to end-to-end learning of CNNs relying on simpler sub-problems and no feedback between layers can scale to large-scale benchmarks such as ImageNet and can be competitive with standard CNN baselines.
We build these competitive models by training only shallow CNNs and using standard architectural elements (ReLU, convolution). Layer-wise training opens the door to applications such as larger models under memory constraints, model prototyping, joint model compression and training, parallelized training, and more stable training for challenging  scenarios.
Importantly, our results suggest a number of open questions regarding the mechanisms that underlie the success of CNNs and provide a major simplification for theoretical research aiming at analyzing high performance deep learning models. Future work can study whether the 1-hidden layer cascades objective can be better specified to more closely mimic the $k>1$ objective.  
%\subsection*{Acknowledgements} We would like to thank Kyle Kastner, John Zarka, Louis Thiry, Georgios Exarchakis, Maxim Berman, Amal Rannen, Florent Perronnin for helpful discussions. EB is partially funded by a Google Focused Research Awards.
%Can the 1-hidden layer network objective be better specified, closing the gap to $(k>1)$?
%for example can the 1-hidden layer network objective be better specified, filling in the gap between 1-hidden layer network and the $k>1$ approach?  
%How can we explain the performance of our 1-hidden layer network?

%list some ideas for further scaling
%\section{Conclusion}
%\label{sec:conc}

%\bibliographystyle{iclr2019_conference}
\bibliographystyle{icml2019}
\bibliography{bibliography_for_papers}
\newpage
\clearpage

\appendix

\section{Proof of Proposition }\label{appendix:proof_prop_opt}

\setcounter{section}{3}
\setcounter{prop}{0}
\renewcommand{\theprop}{\arabic{section}.\arabic{prop}}\begin{prop}[Progressive improvement] Assume that $P_j=Id$. Then there exists $\theta_0$ such that:
\[\hat{\mathcal{R}}(z_{j+1};\theta_{j}^*,\gamma_{j}^*)\leq\hat{\mathcal{R}}(z_{j+1};\theta_0,\gamma_{j-1}^*)=\hat{\mathcal{R}}(z_{j};\theta_{j-1}^*,\gamma_{j-1}^*)\,.\]\end{prop}
\begin{proof}
As $\rho(\rho( x)) = \rho(x)$, we simply have to chose $\theta_0$ such that $W_{\theta_0}=Id$.
%At depth $j$, observe that in the setting of the paper, we can find a parameter $\theta_j$ such that: $\rho W_{\theta_j}=I$ and if we set $C_{\gamma_{j+1}}=C_{\gamma_j^*}$, then the loss at depth $j+1$ is equal to: $\mathcal{R}_j^*$ and thus $\mathcal{R}^*_{j+1}\leq \mathcal{R}^*_{j}$.
\end{proof}

We will now show that given an optimization  $\epsilon$-optimal procedure for the sub-problem optimization, the optimization of Algo 1. can be used directly to obtain an error on the overall. solution. Denote the parameters $\{\theta_{1}^*,...,\theta_{J}^*\}$ the optimal solutions for Algo 1.
\begin{prop}
Assume the parameters $\{\theta_{0}^*,...,\theta_{J-1}^*\}$ are obtained via a optimal layerwise optimization procedure. We assume that $ W_{\theta_j^*}$ is 1-lipschitz without loss of generality and that the biases are bounded uniformly by $B$. Given an input function $g(x)$, we consider functions of the type $z_{g}(x)=C_\gamma \rho W_\theta g(x)$. For $\epsilon>0$, we call $\theta_{\epsilon,g}$ the parameter provided by a procedure to minimize $\hat{\mathcal{R}}(z_{g};\theta;\gamma)$ which leads to a 1-lipschitz operator that satisfies:
1.~$ \underbrace{\Vert \rho W_{\theta_{\epsilon,g}}g(x) -\rho W_{\theta_{\epsilon,\tilde g}}\tilde{g}(x)\Vert\leq\Vert g(x)-\tilde{g}(x)\Vert}_\text{ (stability)},\forall g,\tilde g,\quad$ \\ 2.~~~~~~~$\underbrace{\Vert W_{\theta_j^*} x_j^*-W_{\theta_{\epsilon,x_j^*}} x_j^*\Vert\leq\epsilon(1+\Vert  x_j^* \Vert)}_\text{ ($\epsilon$-approximation)},$

with, $\hat x_{j+1}=\rho W_{\theta_{\epsilon,\hat x_j}}\hat x_j$ and $ x_{j+1}^*=\rho W_{\theta_j^*} x_j^*$ with $x_0^*=\hat x_0=x$, then, we prove by induction:\begin{equation}
\Vert  x_J^*-\hat x_{J}\Vert=\mathcal{O}(J^2\epsilon) \end{equation}\end{prop}

\begin{proof}
First observe that $\Vert x^*_{j+1}\Vert\leq \Vert x_j^* \Vert + B$ by non expansivity. Thus, by induction, $\Vert x_j^*\Vert \leq jB + \Vert x \Vert$. Then, let us show that: $\Vert  x_j^*-\hat x_j \Vert \leq \epsilon (\frac{j(j-1)}{2}B+j\Vert x\Vert+j)$ by induction. Indeed, for $j+1$:% and $\Vert \tilde x_j\Vert\leq \Vert x \Vert$  by non-expansivity; then at depth $j+1$:

\begin{align*}
\Vert &x^*_{j+1}-\hat x_{j+1}\Vert\\
& =\Vert \rho W_{\theta_j^*}x_j^*-\rho W_{\theta_{\epsilon,\hat x_j}} \hat x_j \Vert \\
& =  \Vert \rho W_{\theta_j^*}x_j^*-\rho W_{\theta_{\epsilon,x_j^*}} x_j^*+\rho W_{\theta_{\epsilon,x_j^*}} x_j^*-\rho W_{\theta_{\epsilon,\hat x_j}} \hat x_j \Vert\\
&\text{ by non-expansivity}\\
& \leq   \Vert  W_{\theta_j^*}x_j^*- W_{\theta_{\epsilon,x_j^*}} x_j^* \Vert + \Vert \rho W_{\theta_{\epsilon,x_j^*}} x_j^*-\rho W_{\theta_{\epsilon,\hat x_j}} \hat x_j \Vert \\
& \leq  \epsilon \Vert x_j^*\Vert + \epsilon + \Vert  x_j^*-\hat x_j \Vert\text{ from the assumptions}\\
& \leq  \epsilon (jB+\Vert x\Vert+1)+ \Vert  x_j^*-\hat x_j \Vert\text{ from above} \\
&\text{ by induction}\\
& \leq  \epsilon (jB+\Vert x\Vert+1)+ \epsilon (\frac{j(j-1)}{2}B+j+j\Vert x\Vert) \Vert \\
&=   \epsilon (\frac{j(j+1)}{2}B+(j+1)\Vert x\Vert+(j+1)) 
\end{align*}
%& =  \Vert W_{\theta_j^*}x_j^*-W_{\theta_j^*}\tilde x_j+W_{\theta_j^*}\tilde x_j-W_{\theta_{\epsilon,x_j^*}} \tilde x_j+W_{\theta_{\epsilon,x_j^*}} \tilde x_j-W_{\theta_{\epsilon,\tilde x_j}} \tilde x_j\Vert\\
%& \leq  \Vert W_{\theta_j^*}x_j^*-W_{\theta_j^*}\tilde x_j\Vert +\Vert W_{\theta_j^*}\tilde x_j-W_{\theta_{\epsilon,x_j^*}} \tilde x_j\Vert + \Vert W_{\theta_{\epsilon, x_j^*}} \tilde x_j-W_{\theta_{\epsilon,\tilde x_j}} \tilde x_j\Vert\\
%& \leq  \Vert W_{\theta_j^*}\Vert\Vert x_j^*-\tilde x_j \Vert+\Vert W_{\theta_j^*}-W_{\theta_{\epsilon, x_j^*}}\Vert\Vert \tilde x_j\Vert+\Vert W_{\theta_{\epsilon, x_j^*}} -W_{\theta_{\epsilon,\tilde x_j}} \Vert\Vert \tilde x_j\Vert\\
%& \leq  \Vert x_j^*-\tilde x_j\Vert+\Vert W_{\theta_j^*}-W_{\theta_{\epsilon, x_j^*}}\Vert \Vert x \Vert +\Vert W_{\theta_{\epsilon, x_j^*}} -W_{\theta_{\epsilon,\tilde x_j}}\Vert \Vert x \Vert\text{, again by non-expansivity} \\
%& \leq  \Vert x_j^*-x_j^\epsilon\Vert+\epsilon \Vert x \Vert +\Vert x_j^*-x_j^\epsilon\Vert \\
%& =\epsilon \Vert x \Vert +2\Vert x_j^*-\tilde x_j\Vert\\
%& \leq \epsilon \Vert x \Vert + 2( 2^{j}-1)\epsilon \Vert x \Vert=(2^{j+1}-1)\epsilon \Vert x \Vert
As $x^*_0=\hat x_0=x$,  the property is true for $j=0$.

\end{proof}

\subsection*{A Note on Sec.2 From \cite{mallat2016understanding}}
We first briefly discuss the result of Sec.2 of \cite{mallat2016understanding}. Let us introduce: $\Omega_j = \{x : \forall x', f(x)\neq f(x') \Rightarrow \rho W_{\hat\theta_{j-1}}\cdots \rho W_{\hat\theta_{1}}x \neq \rho W_{\hat\theta_{j-1}} \cdots \rho W_{\hat\theta_{1}}x'\}$. We introduce: $\mathcal{Y}_j=\{y : \exists x \in \Omega_j, y=\rho W_{\hat\theta_{j-1}}\cdots \rho W_{\hat\theta_{1}}x\}$. For $y\in \mathcal{Y}_j$, we define $\hat{f}_j(y)\triangleq f(x)$. Observe this defines indeed a function, and that:
\[\forall x\in \Omega_j, f(x)=\hat{f}_j\circ \rho W_{\hat\theta_{j-1}}\cdots\rho W_{\hat\theta_{1}}x\]

Observe also that $\Omega_{j+1}\subset \Omega_j$. The set $\Omega_j$ is simply the set of samples which are well discriminated by the neural network $\rho W_{\hat\theta_{j-1}}\cdots\rho W_{\hat\theta_{1}}$.

\setcounter{section}{1}

\section{Additional Details on Imagenet Models and Performance}
\label{app:layerwise}
For ImageNet we report the improvement in accuracy obtained by adding layers in Figure~\ref{fig:prgnlin} as seen by the auxiliary problem solutions. We observe that indeed the accuracy of the model on both the training and validation is able to improve from adding layers as discussed in depth in Section~\ref{sec:prog_exp}. We observe that $k=1$ also over-fits substantially, suggesting better regularization can help in this setting.
\begin{figure}
    \centering
    \includegraphics[scale=0.35]{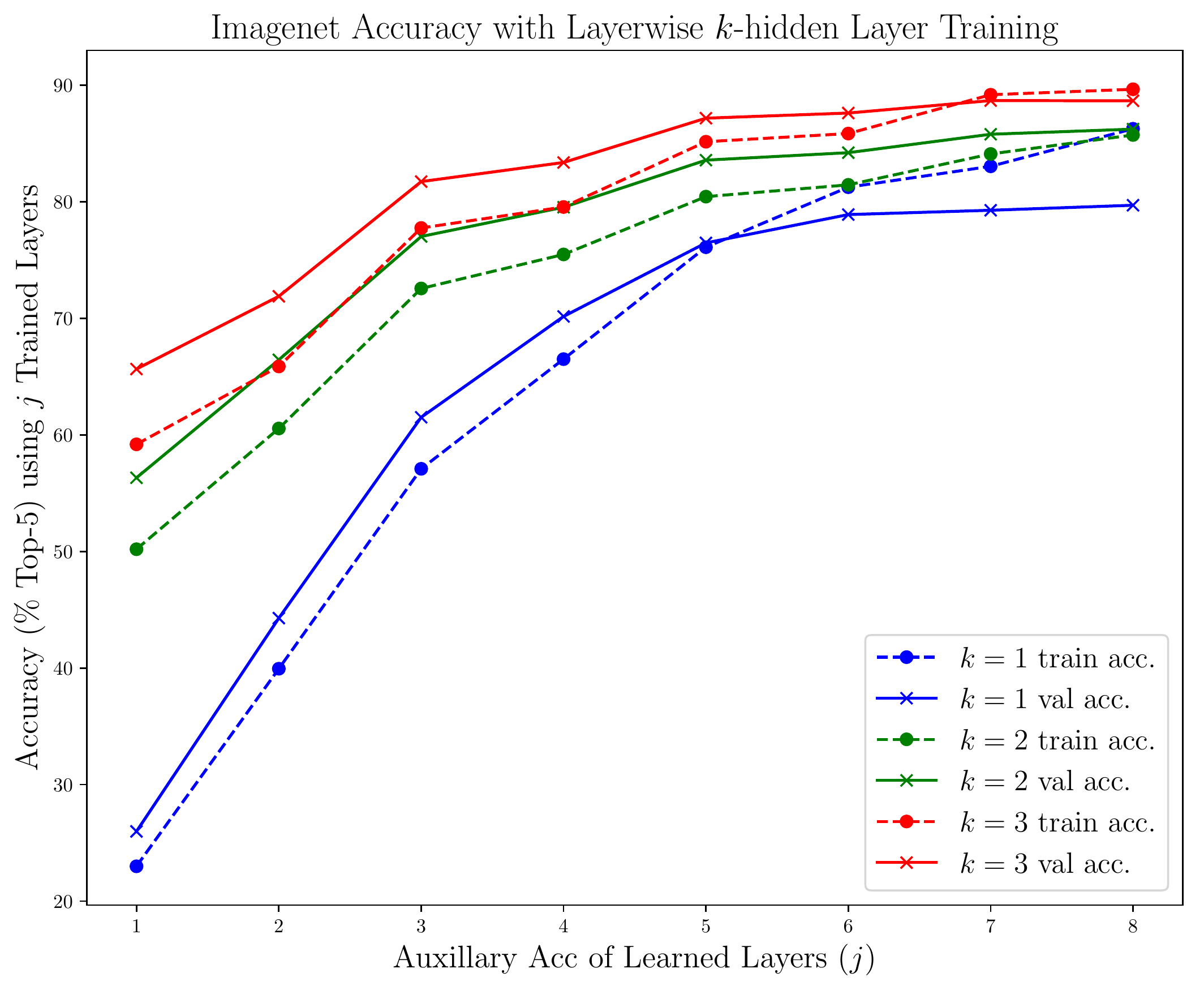}
    \caption{Intermediate Accuracies of models in Sec. ~\ref{sec:scaling}. We note that the $k=1$ model is larger than the $k=2,3$ models.}
    \label{fig:prgnlin}
\end{figure}

We provide a more explicit view of the network sizes in Tab. \ref{tab:nlin12} and Tab. \ref{tab:nlin0}. We also show the number of parameters in the ImageNet networks in Tab. \ref{tab:imagenet_params}. Although some of the models are not as parameter efficient compared to the related ones in the literature, this was not a primary aim of the investigation in our experiments and thus we did not optimize the models for parameter efficiency (except explicitly at the end of Sec. \ref{sec:scaling}), choosing our construction scheme for simplicity. We highlight that this is not a fundamental problem in two ways: (a) for the $k=1$ model we note that removing the last two layers reduces the size by $1/4$, while the top 5 accuracy at the earlier J=6 layer is $78.8$ (versus $79.7$), see Figure~\ref{fig:prgnlin} for detailed accuracies. (b) Our models for $k=2,3$ have most of their parameters in the final auxiliary network which is easy to correct for once care is applied to this specific point as at the end of Sec.~\ref{sec:scaling}. We note also that the model with $k=3,M_f=512$, is actually more parameter efficient than those in the VGG family while having similar performance. We also point out that we use for simplicity the VGG style construction involving only $3x3$ convolutions and downsampling operations that only half the spatial resolution, which indeed has been shown to lead to relatively less parameter efficient architectures \cite{he2016deep}, using less uniform construction (larger filters and bigger pooling early on) can yield more parameter efficient models.

\begin{table}[]
    \centering
    \begin{tabular}{c|c|c}
        Layer & spatial size & layer output size \\\hline
            Input & $112\times112$ & 12\\\hline
            1  & $112\times112$& 128\\\hline
            2  & $112\times112$& 128\\\hline
            3& $56\times56$ & 256\\\hline
            4  & $56\times56$& 256\\\hline
            5  & $28\times28$& 512\\\hline
            6 & $28\times28$ & 512\\\hline
            7  & $14\times14$& 1024\\\hline
            8  & $14\times14$& 1024
    \end{tabular}
    \caption{Network structure for $k=2,3$ imagenet models, not including auxiliary networks. Note an invertible downsampling is applied on the input 224x224x3 image to producie the initial input. The default auxillary networks for both have $\tilde{M}_f=2048$ with 1 and 2 auxillary layers, respectively. Note auxiliary networks always reduce the spatial resolution to $2x2$ before the final linear layer.}
    \label{tab:nlin12}
\end{table}
\begin{table}[]
    \centering
    \begin{tabular}{c|c|c}
        Layer & spatial size & layer output size \\\hline
            Input & $112\times112$ & 12\\\hline
            1  & $112\times112$& 256\\\hline
            2  & $112\times112$& 256\\\hline
            3& $56\times56$ & 512\\\hline
            4  & $28\times28$& 1024\\\hline
            5  & $14\times14$& 2048\\\hline
            6 & $14\times14$ & 2048\\\hline
            7  & $7\times7$& 4096\\\hline
            8  & $7\times7$& 4096
    \end{tabular}
    \caption{Network structure for $k=1$ ImageNet models, not including auxiliary networks. Note an invertible down-sampling is applied on the input 224x224x3 image to produce the initial input. Note this network does not include any batch-norm.}
    \label{tab:nlin0}
\end{table}
\begin{table}
\centering
\begin{tabular}{|c|c|}
\hline

\textbf{\makecell{Layer-wise Trained}}& \makecell{  Acc.\\ }\\\hline
    Strided Convolution & \makecell{  87.8} \\\hline
    Invertible Down &   \makecell{88.3}\\\hline
    AvgPool &  \makecell{ 87.6}\\\hline
    MaxPool &  \makecell{ 88.0}\\\hline
\end{tabular}
  \caption{Comparison of different downsampling operations}
  \label{app:pool}
\end{table}
\begin{table}
\begin{tabular}{|c|c|c|}
 \hline
Models  & Number of Parameters \\\hline
    SimCNN $k=3$, $M_f=512$ & 46M \\\hline
    SimCNN $k=3$ & 102M \\\hline
    SimCNN $k=2$ & 64M \\\hline
  %  Our model $k=1, J=8$&  406M \\\hline
    SimCNN $k=1, J=6$&  96M \\\hline
        AlexNet  & 60M \\\hline
    VGG-16  &  138 M\\\hline
 
 \end{tabular}\caption{Overall parameter counts for SimCNN models trained in Sec. ~\ref{sec:exp} and from literature.}\label{tab:imagenet_params}
\end{table}
\section{Additional Studies}
We report additional studies that elucidate the critical components of the system and demonstrate the transferability properties of the greedily learned features.

\subsection{Choice of Downsampling}
%As mentioned, appropriately dealing with spatial resolution reduction and pooling is an important technical challenge to permit layerwise CNN learning. Indeed pooling will tend to create invariances which cannot be easily accounted for, unlike in end to end learning where the model 

In our experiments we use primarily the invertible downsampling operator as the downsampling operation. This choice is to reduce architectural elements which may be inherently lossy such as average pooling. Compared to maxpooling operations it also helps to maintain the network in Sec.~\ref{sec:alex} as a pure ReLU network, which may aid in analysis as the maxpooling introduces an additional non-linearity. We show here the effects of using alternative downsampling approaches including: average pooling, maxpooling and strided convolution. On the CIFAR dataset in the setting of $k=1$ we find that they ultimately lead to very similar results with invertible downsampling being slightly better. This shows the method is rather general. In our experiments we follow the same setting described for CIFAR. The setting here uses $J=5$ and downsamplings at $j=2,3$. The size is always halved in all cases and the downsampling operation and the output sizes of all networks are the same. Specifically the Average Pooling and Max Pooling use $2\times2$ kernels and the strided convolution simply modifies the $3\times3$ convolutions in use to have a stride of $2$. Results are shown in Tab. ~\ref{app:pool}.

\subsection{Effect of Width}
\label{app:cifar_exp}

We report here an additional view of the aggregated results for linear separability discussed in Sec.~\ref{sec:prog_exp}. We observe that the trend of the aggregated diagram is similar when comparing only same sized models, with the primary differences in model sizes being increased accuracy.

\begin{figure}
    \centering
    \includegraphics[width=0.9\linewidth]{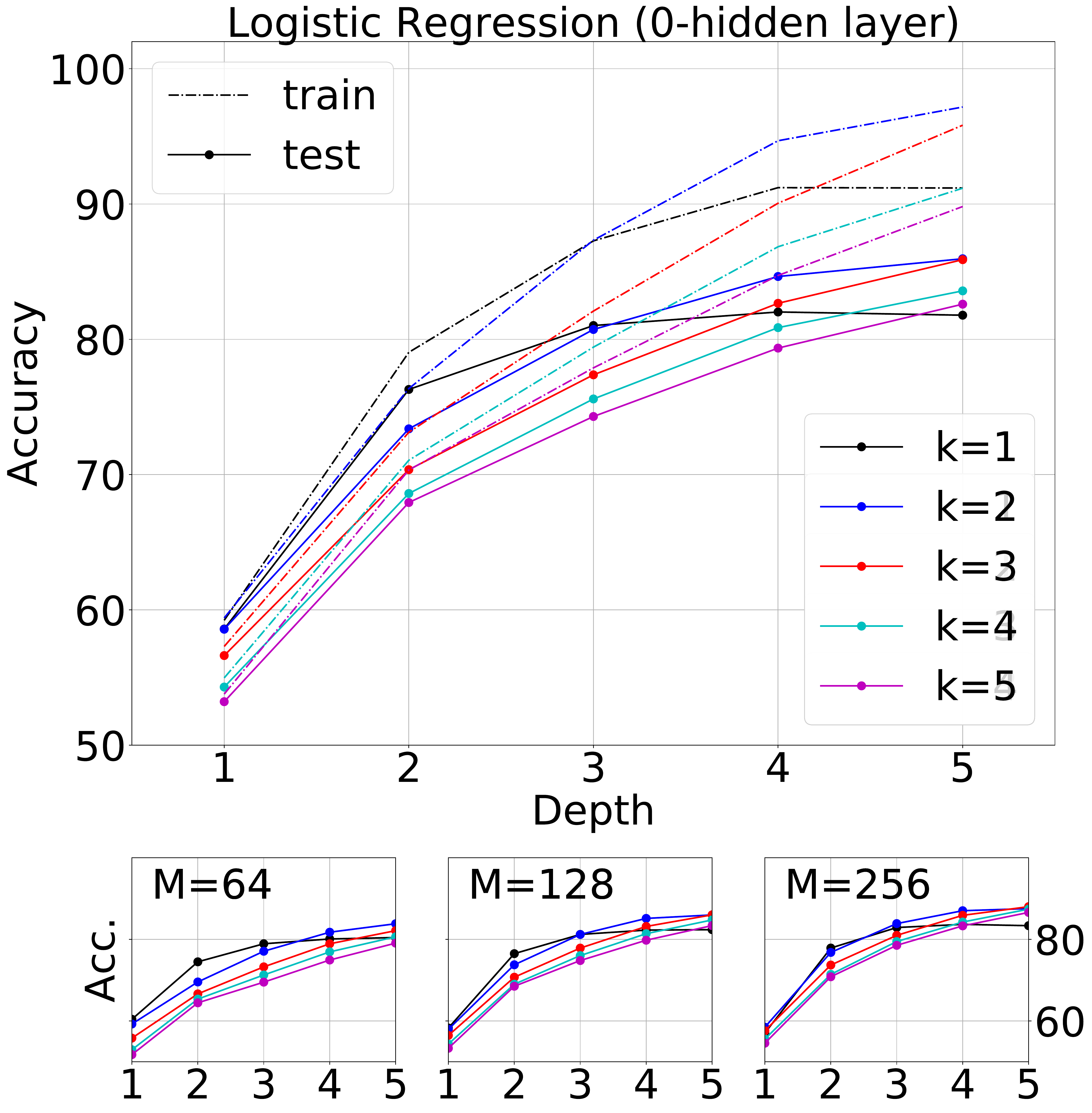}
    \caption{Linear separability of differently trained sequential models. We show how the data varies for the different $M$, observing similar trends to the aggregated data.}
    \label{fig:layer_size}
\end{figure}

\subsection{Transfer Learning on Caltech-101}
\label{app:transfer}
Deep CNNs such as AlexNet trained on Imagenet are well known to have generic properties for computer vision tasks, permitting transfer learning on many downstream applications. We briefly evaluate here if the  $k=1$ imagenet model (Sec. ~\ref{sec:alex}) shares this generality on the  Caltech-101 dataset. This dataset has 101 classes and we follow the same standard experimental protocol as \cite{zeiler2014visualizing}: 30 images per class are randomly selected, and the rest is used for testing. The average per class accuracy is reported using 10 random splits. As in \cite{zeiler2014visualizing} we restrict ourselves to a linear model. We use a multinomial logistic regression applied on features from different layers including the final one. For the logistic regression we rely on the default hyperparameter settings for logistic regression of the \texttt{sklearn} package using the SAGA algorithm. We apply a linear averaging and PCA transform (for each fold) to reduce the dimensionality to $500$ in all cases. We find the results are similar to those reported in \cite{zeiler2014visualizing} for their version of the AlexNet. This highlights the model has similar transfer properties and also shows similar progressive linear separability properties as end-to-end trained models. 

\begin{table}[]
    \centering
    \begin{tabular}{|c|c|}
        \hline
                & Accuracy \\\hline
        \makecell{ConvNet from scratch \\ \citep{zeiler2014visualizing}}  & $46 \pm 1.7\%$ \\\hline         
         Layer 1&  $45.5 \pm 0.9$\\\hline
         Layer 2&  $59.9     \pm 0.9 $\\\hline
         Layer 3&  $70.0      \pm  0.9  $\\\hline
         Layer 4& $75.0 \pm 1.0$\\\hline
         Layer 8&  $82.6 \pm 0.9$ \\\hline
    \end{tabular}
    \caption{Accuracy obtained by a linear model using the features of the $k=1$ network at a given layer on the Caltech-101 dataset. We also give the reference accuracy without transfer.}
    \label{tab:caltech}
\end{table}

\end{document}